\documentclass[onecolumn,11pt,letter]{article}
\usepackage{mathptmx}
\usepackage{epsfig}
\usepackage{epstopdf}
\usepackage{graphicx}
\usepackage{amsfonts}
\usepackage{amsmath}
\usepackage{amsthm}
\usepackage{proof}
\usepackage{latexsym}
\usepackage{amssymb}
\usepackage{color}
\usepackage[normalem]{ulem}
\usepackage{comment}
\usepackage[ruled]{algorithm2e}
\usepackage[margin=1in]{geometry}

\newtheorem{theorem}{Theorem}[section]
\newtheorem{lemma}[theorem]{Lemma}

\newtheorem{definition}[theorem]{Definition}

{\makeatletter
 \gdef\xxxmark{%
   \expandafter\ifx\csname @mpargs\endcsname\relax 
     \expandafter\ifx\csname @captype\endcsname\relax 
       \marginpar{xxx}
     \else
       xxx 
     \fi
   \else
     xxx 
   \fi}
 \gdef\xxx{\@ifnextchar[\xxx@lab\xxx@nolab}
 \long\gdef\xxx@lab[#1]#2{{\bf [\xxxmark #2 ---{\sc #1}]}}
 \long\gdef\xxx@nolab#1{{\bf [\xxxmark #1]}}
 \long\gdef\xxx@lab[#1]#2{}\long\gdef\xxx@nolab#1{}%
}

\begin{document}

\newcommand{\Mb}[1]{\mathbf{#1}}
\newcommand{\Mc}[1]{\mathcal{#1}}

\title{Parameterized Complexity of Problems in Coalitional Resource Games\thanks{Supported in part by Google Faculty Research Award, ONR Young
Investigator Award and NSF CAREER award.}}

\author{Rajesh Chitnis\thanks{Department of Computer Science , University of Maryland at College Park, USA, email: rchitnis@cs.umd.edu}
\and MohammadTaghi Hajiaghayi\thanks{Department of Computer Science , University of Maryland at College Park, USA. email:
hajiagha@cs.umd.edu} \and Vahid Liaghat\thanks{Department of Computer Science , University of Maryland at College Park, USA.
email: vliaghat@cs.umd.edu}}

\maketitle

\begin{abstract}
Coalition formation is a key topic in multi-agent systems. Coalitions enable agents to achieve goals that they may not have
been able to achieve on their own. Previous work has shown problems in coalition games to be computationally hard. Wooldridge
and Dunne (Artificial Intelligence 2006) studied the classical computational complexity of several natural decision problems
in Coalitional Resource Games (CRG) - games in which each agent is endowed with a set of resources and coalitions can bring
about a set of goals if they are collectively endowed with the necessary amount of resources.
The input of coalitional resource games bundles together several elements, e.g., the agent set $Ag$, the goal set $G$, the
resource set $R$, etc. Shrot, Aumann and Kraus (AAMAS 2009) examine coalition formation problems in the CRG model using the
theory of Parameterized  Complexity. Their refined analysis shows that not all parts of input act equal - some instances of
the problem are indeed tractable while others still remain intractable.

We answer an important question left open by Shrot, Aumann and Kraus by showing that the SC Problem (checking whether a
Coalition is Successful) is W[1]-hard when parameterized by the size of the coalition. Then via a single theme of reduction
from SC, we are able to show that various problems related to resources, resource bounds and resource conflicts introduced by
Wooldridge et al are
\begin{enumerate}
\item W[1]-hard or co-W[1]-hard when parameterized by the size of the coalition.
\item para-NP-hard or co-para-NP-hard when parameterized by $|R|$.
\item FPT when parameterized by either $|G|$ or $|Ag|+|R|$.
\end{enumerate}

%
\end{abstract}

\newpage

\section{Introduction}
\label{introduction}

\subsection{Coalitions}
\label{coalitions} In multi-agent systems (MAS), where each agent has limited resources, the formation of coalitions of agents
is a very powerful tool~\cite{Wooldridge-book}. Coalitions enable agents to accomplish goals they may not have been able to
accomplish individually. As such, understanding and predicting the dynamics of coalitions formation, e.g., which coalitions
are more beneficial and/or more likely to emerge, is a question of considerable interest in multi-agent settings.
Unfortunately, a range of previous studies have shown that many of these problems are computationally
complex~\cite{woolridge-qcg,woolbridge-crg}. Nonetheless, as noted by Garey and Johnson~\cite{garey-johnson}, hardness
results, such as NP-completeness, should merely constitute the beginning of the research. NP-hardness just indicates that a
general solution for all instances of the problem most probably does not exist. Still, efficient solutions for important
sub-classes may well exist.

\subsection{Formal Model of Coalition Resource Games}
\label{formal-model-of-crg} The framework we use to model coalitions is the CRG model introduced in \cite{woolbridge-crg},
defined as follows. The model contains a non-empty, finite set $Ag = \{a_1,\ldots, a_n\}$ of \emph{agents}. A
\emph{coalition}, typically denoted by~$C$, is simply a set of agents, i.e., a subset of $Ag$. The \emph{grand coalition} is
the set of all agents,~$Ag$. There is also a finite set of \emph{goals}~$G$. Each agent $i\in Ag$ is associated with a
subset~$G_i$ of the goals. Agent~$i$ is \emph{satisfied} if at least one member of~$G_i$ is achieved, and unsatisfied
otherwise. Achieving the goals requires the expenditure of \emph{resources}, drawn from the total set of resource types~$R$.
Achieving different goals may require different quantities of each resource type. The quantity $\textbf{req}(g, r)$ denotes
the amount of resource~$r$ required to achieve goal~$g$. It is assumed that $\textbf{req}(g, r)$ is a non-negative integer.
Each agent is \emph{endowed} certain amounts of some or all of the resource types. The quantity $\textbf{en}(i, r)$ denotes
the amount of resource~$r$ endowed to agent~$i$. Again, it is assumed that $\textbf{en}(i, r)$ is a non-negative integer.
Formally, a \emph{Coalition Resource Game} $\Gamma$ is a $(n+5)$-tuple given by
$$ \Gamma = \langle Ag,G,R,G_1,G_2,\ldots,G_n,\textbf{en},\textbf{req} \rangle $$
where:
\begin{itemize}
  \item $Ag = \{a_1,a_2,\ldots,a_n\}$ is the set of agents
  \item $G = \{g_1,g_2,\ldots,g_m\}$ is the set of possible goals
  \item $R = \{r_1,r_2,\ldots,r_t\}$ is the set of resources
  \item For each $i\in Ag$, $G_i$ is a subset of $G$ such that any of the goals in~$G_i$ would satisfy~$i$
	  but $i$~is indifferent between the members of $G_i$
  \item \textbf{en} : $Ag \times R \rightarrow \mathbb N \cup \{0\}$ is the endowment function
  \item \textbf{req} : $G \times R \rightarrow \mathbb N \cup \{0\}$ is the requirement function
\end{itemize}
The endowment function \textbf{en} extends to coalitions by summing up endowments of its members as
\begin{center}
$ \textbf{en}(C,r) = \sum_{i\in C} \textbf{en}(i,r) \ \hspace{10mm} \forall r\in R $
\end{center}
\noindent The requirement function \textbf{req} extends to sets of goals by summing up requirements of its members as
\begin{center}
$ \textbf{req}(G',r) = \sum_{g\in G'} \textbf{req}(g,r) \ \hspace{5mm} \forall r\in R $
\end{center}
\noindent A set of goals~$G'$ \emph{satisfies} agent~$i$ if $G_i \cap G' \neq \emptyset$ and satisfies a coalition~$C$ if it
satisfies every member of~$C$. A set of goals~$G'$ is \emph{feasible} for coalition~$C$ if that coalition is endowed with
sufficient resources to achieve all goals in~$G'$, i.e., for all $r\in R$ we have $\textbf{req}(G',r)\leq \textbf{en}(C,r)$.
Finally we say that a coalition $C$ is \emph{successful} if there exists a non-empty set of goals~$G'$ that satisfies~$C$ and
is feasible for it. In general, we use the notation $ succ(C) = \{G'\ |\ G'\subseteq G,\ G'\neq \emptyset$ and $G'$ is
successful for $C\}$. The CRG models many real-world situations like the virtual organizations problem~\cite{conitzer-virtual}
and voting domains.

\section{Problem Definitions and Previous Work}

\subsection{Problems Related to Coalition Formation}
\label{summary-of-kraus} Shrot et al.~\cite{kraus} considered the following four problems related to coalitions.
\begin{enumerate}
	\item \textsc{Successful Coalition} (SC)\\
			 Instance: A CRG $\Gamma$ and a coalition~$C$\\
			 Question: Is $C$ successful?	
	\item \textsc{Exists a Successful Coalition of Size $k$} (ESCK)\\
			 Instance: A CRG $\Gamma$ and an integer~$k$\\
			 Question: Does there exist a successful coalition of size exactly~$k$?
	\item \textsc{Maximal Coalition} (MAXC)\\
			 Instance: A CRG $\Gamma$ and a coalition $C$\\
			 Question: Is every proper superset of $C$ not successful?
	\item \textsc{Maximal Successful Coalition} (MAXS)\\
			 Instance: A CRG $\Gamma$ and a coalition $C$\\
			 Question: Is $C$ successful and every proper superset of~$C$ not successful?
\end{enumerate}
The results from Shrot et al.~\cite{kraus} are summarized in Figure~\ref{shrot-table}.
\begin{figure}
\begin{center}
    \begin{tabular}{ | l | l | l | l |}
    \hline
     & \textcolor[rgb]{1.00,0.00,0.00}{SC} & \textcolor[rgb]{1.00,0.00,0.00}{ESCK} & \textcolor[rgb]{1.00,0.00,0.00}{MAXC, MAXSC}\\ \hline
    \textcolor[rgb]{1.00,0.00,0.00}{$|G|$} & FPT & FPT & FPT\\ \hline
    \textcolor[rgb]{1.00,0.00,0.00}{$|C|$} & \hspace{10mm}? & W[1]-Hard & W[1]-Hard \\ \hline
    \textcolor[rgb]{1.00,0.00,0.00}{$|R|$} & para-NP-Hard & \hspace{7mm}? & para-NP-Hard \\\hline
    \textcolor[rgb]{1.00,0.00,0.00}{$|Ag|+|R|$} & FPT & \hspace{7mm}? & FPT \\\hline
    \end{tabular}
\end{center}
\caption{Results of Shrot et al.~\cite{kraus}}
\label{shrot-table}
\end{figure}

In this work we consider the problems which were defined by Wooldridge et al.~\cite{woolbridge-crg} but were not considered by
Shrot et al.~\cite{kraus}. We define these problems in detail in the following sections.

\begin{enumerate}
\setcounter{enumi}{4}
	\item \textsc{Necessary Resource} (NR)\\
			 Instance: A CRG $\Gamma$, coalition $C$ and resource $r$\\
			 Question: Is $\textbf{req}(G',r)>0 \ \forall \ G'\in succ(C)$?
    \item \textsc{Strictly Necessary Resource} (SNR)\\
			 Instance: A CRG $\Gamma$, coalition $C$ and resource $r$\\
			 Question: Is $succ(C)\neq \emptyset$  and $\forall \ G'\in succ(C)$ we have $\textbf{req}(G',r)>0$?
    \item \textsc{$(C,G_0,r)$-Optimality} (CGRO)\\
			 Instance: A CRG $\Gamma$, coalition $C$, goal set $G_0\in succ(C)$ and resource $r$\\
			 Question: Is $\textbf{req}(G',r)\geq \textbf{req}(G_0,r) \ \forall \ G'\in succ(C)$?
    \item \textsc{R-Pareto Efficient Goal Set} (RPEGS)\\
			 Instance: A CRG $\Gamma$, coalition $C$ and a goal set $G_0$\\
			 Question: Is $G_0$ \emph{R-Pareto Efficient} for coalition $C$?
	\item \textsc{Successful Coalition With Resource Bound} (SCRB)\\
			 Instance: A CRG $\Gamma$, coalition $C$ and a resource bound~\textbf{b}\\
			 Question: Does $ \exists \ G_0\in succ(C)$ such that $G_0$ respects~\textbf{b}?
    \item \textsc{Conflicting Coalitions} (CC)\\
			 Instance: A CRG $\Gamma$, coalitions $C_1,C_2$ and a resource bound~\textbf{b}\\
			 Question: If $\forall\ G_1\in succ(C_1)$ and $\forall\ G_2\in succ(C_2)$ we have $cgs(G_1,G_2,\textbf{b})$?
\end{enumerate}

\section{Parameterized Complexity}
\label{pc-primer} We now provide a brief introduction to the key relevant concepts from the theory of parameterized
complexity. The definitions in this section are taken from \cite{flum-grohe} and~\cite{downey-skeptic}. The core idea of
parameterized complexity is to single out a specific part of the input as the parameter and ask whether the problem admits an
algorithm that is efficient in all but the parameter. In most cases the parameter is simply one of the elements of the input
(e.g.,\ the size of the goal set), but it can actually be any computable function of the input:
\begin{definition}
Let $\Sigma$ be a finite alphabet.
\begin{enumerate}
  \item  A \textbf{parametrization} of $\Sigma^{*}$
is a \emph{mapping} $\kappa : \Sigma^{*} \rightarrow \mathbb{N}$
that is \emph{polynomial time computable}.
  \item A \textbf{parameterized problem} (over $\Sigma$) is a pair $(Q, \kappa)$
consisting of a set $Q\subseteq \Sigma^{*}$ of strings over~$\Sigma$ and a parameterization~$\kappa$ of~$\Sigma^{*}$.
\end{enumerate}
\end{definition}
As stated, given a parameterized problem we seek an algorithm
that is efficient in all but the parameter. This is
captured by the notion of \emph{fixed parameter tractability}, as follows:
\begin{definition}
A parameterized problem $(Q, \kappa )$ is fixed
parameter tractable (FPT) if there exist an algorithm~$\mathbb{A}$, a
constant~$\alpha$, and a computable function~$f$, such that $\mathbb{A}$ decides
$Q$ in time $f(\kappa(x))|x|^{\alpha}$.
\end{definition}
Thus, while the fixed-parameter notion allows inefficiency in the parameter $\kappa (x)$, by means of the function~$f$, it
requires polynomial complexity in all the rest of the input. In particular, a problem that is FPT is tractable for any bounded
parameter value. While the core aim of parameterized complexity is to identify problems that are fixed-parameter tractable, it
has also developed an extensive complexity theory, allowing to prove hardness results, e.g.,\ that certain problems are (most
probably) not FPT. To this end, several parameterized complexity classes have been defined. Two of these classes are the class
W[1] and the class para-NP. We will formally define these classes shortly, but the important point to know is that there is
strong evidence to believe that both classes are not contained in FPT (much like NP is probably not contained in P). Thus,
W[1]-hard and para-NP-hard problems are most probably not fixed-parameter tractable. The class W[1] can be defined by its core
complete problem, defined as follows:
~\\

\hspace{-6mm}\textsc{\textbf{Short Nondeterministic Turing Machine Computation}}\\
Instance: A single-tape, single-head non-deterministic Turing machine~$M$, a word~$x$ and an integer~$k$\\
Question: Is there a computation of~$M$ on input~$x$ that reaches the accepting state in at most $k$~steps?\\
Parameter: $k$\\

Note that this definition is analogous to that of NP, with
the addition of the parameter~$k$.
\begin{definition}
The class \textup{W[1]} contains all parameterized problems FPT-reducible (defined hereunder) to Short-Nondeterministic-Turing-Machine-Computation.
\end{definition}
The class para-NP is defined as follows :
\begin{definition}
A parameterized problem $(Q, \kappa )$ is in para-NP
if there exists a non-deterministic Turing machine~$M$,
constant~$\alpha$ and an arbitrary computable function~$f$, such that
for any input~$x$, $M$~decides if $x\in Q$ in time $\leq |x|^{\alpha}f(\kappa(x))$.
\end{definition}
Establishing hardness results most frequently requires reductions.
In parameterized complexity, we use FPT-reduction,
defined as follows:
\begin{definition}
Let $(Q, \kappa )$ and $(Q', \kappa'
)$ be parameterized
problems over the alphabets $\Sigma$ and $\Sigma'$
respectively. An FPT-reduction (FPT many-to-one reduction) from $(Q, \kappa )$ to $(Q', \kappa'
)$ is a mapping $ R: \Sigma^{*}\rightarrow (\Sigma')^{*} $ such that:
\begin{enumerate}
  \item For all $x\in \Sigma^{*}$ we have $x\in Q \Leftrightarrow R(x)\in Q'$.
  \item $R$ is computable in time $f(\kappa (x))|x|^{\alpha}$ for some constant
$\alpha$ and an arbitrary function~$f$.
  \item There is a computable function $g : \mathbb{N} \rightarrow \mathbb{N}$ such that
$\kappa'(R(x)) \leq g(\kappa(x))$ for all $x\in \Sigma^{*}$.
\end{enumerate}
\end{definition}
Point~(1) simply states that $R$ is indeed a reduction. Point~(2)
says that it can be computed in the right amount of time
- efficient in all but the parameter. Point~(3) states that
the parameter of the image is bounded by (a function of)
that of the source. This is necessary in order to guarantee
that FPT-reductions preserve FPT-ness, i.e.\ with this
definition we obtain that if $(Q, \kappa )$ reduces to $(Q', \kappa')$ and
$(Q', \kappa')\in $ FPT then $(Q, \kappa)$ is also in FPT.

\section{Our Results \& Techniques}

We consider problems regarding resources bounds and resource conflicts which were shown to be computationally hard in
Wooldridge et al. (\cite{woolbridge-crg}) but were not considered in Shrot et al.~\cite{kraus}. We also solve three open
questions posed in Shrot et al. by showing that
 \begin{enumerate}
   \item SC parameterized by $|C|$ is \emph{W[1]-hard}
   \item ESCK parameterized by $|Ag|+|R|$ is \emph{FPT}
   \item ESCK parameterized by $|R|$ is \emph{para-NP-hard}
 \end{enumerate}
We study the complexity of NR, SNR, CGRO, RPEGS, SCRB and CC problems when parameterized by natural parameters $|G|, |C|,|R|$
and $|Ag|+|R|$. We also give a general integer program which with slight modifications for each problem shows that these
problems are FPT when parameterized by $|G|$ or $|Ag|+|R|$ (except CC parameterized by $|Ag|+|R|$ which is open). We note that
Shrot et al. showed that SC parameterized by $|R|$ is para-NP-hard. We complete this hardness result by showing that SC
parameterized by $|C|$ is W[1]-hard and thus answer their open question. Using these hardness results and via a single theme
of parameter preserving reductions we show that hardness results for all of the above problems when parameterized by $|R|$ and
$|C|$. We also show that Theorem 3.2 of Shrot et al.~\cite{kraus} is false - which claims that ESCK is FPT when parameterized
by $|G|$. We give a counterexample to their proposed algorithm and show that the problem is indeed para-NP-hard.

These results help us to understand the role of various components of the input and identify which ones actually make the
input hard. Since all the problems we considered remain intractable when parameterized by $|C|$ or $|R|$, there is no point in
trying to restrict these parameters. On the other hand, most of the problems are FPT when parameterized by $|G|$ or $|Ag|+|R|$
and thus we might enforce this restriction in real-life situations to ensure the tractability of these problems.

\begin{figure*}
\begin{center}
    \begin{tabular}{ | l | l | l | l | l | l | l | }
    \hline
     & \textcolor[rgb]{0.00,0.00,1.00}{SC} & \textcolor[rgb]{0.00,0.00,1.00}{ESCK} & \textcolor[rgb]{0.00,0.00,1.00}{NR,CGRO,RPEGS } & \textcolor[rgb]{0.00,0.00,1.00}{SNR} & \textcolor[rgb]{0.00,0.00,1.00}{SCRB} & \textcolor[rgb]{0.00,0.00,1.00}{CC}\\ \hline
     & \textcolor[rgb]{0.00,1.00,0.00}{NPC} & \textcolor[rgb]{0.00,1.00,0.00}{NPC} & \textcolor[rgb]{0.00,1.00,0.00}{co-NPC} & \textcolor[rgb]{0.00,1.00,0.00}{$D^{p}$-complete} & \textcolor[rgb]{0.00,1.00,0.00}{NPC} & \textcolor[rgb]{0.00,1.00,0.00}{co-NPC}\\ \hline
    \textcolor[rgb]{0.00,0.00,1.00}{$|G|$} & FPT & \sout{FPT} \textcolor[rgb]{1.00,0.00,0.00}{p-NP-hard} & \textcolor[rgb]{1.00,0.00,0.00}{FPT} & \textcolor[rgb]{1.00,0.00,0.00}{FPT} & \textcolor[rgb]{1.00,0.00,0.00}{FPT} & \textcolor[rgb]{1.00,0.00,0.00}{FPT}\\ \hline
    \textcolor[rgb]{0.00,0.00,1.00}{$|C|$} & \textcolor[rgb]{1.00,0.00,0.00}{W[1]-hard} & W[1]-hard & \textcolor[rgb]{1.00,0.00,0.00}{co-W[1]-hard} & \textcolor[rgb]{1.00,0.00,0.00}{W[1]-hard} & \textcolor[rgb]{1.00,0.00,0.00}{co-W[1]-hard} & \textcolor[rgb]{1.00,0.00,0.00}{co-W[1]-hard} \\\hline
    \textcolor[rgb]{0.00,0.00,1.00}{$|R|$} & p-NP-hard & \textcolor[rgb]{1.00,0.00,0.00}{p-NP-hard} & \textcolor[rgb]{1.00,0.00,0.00}{co-pNP-hard} & \textcolor[rgb]{1.00,0.00,0.00}{pNP-hard} & \textcolor[rgb]{1.00,0.00,0.00}{co-pNP-hard} & \textcolor[rgb]{1.00,0.00,0.00}{co-pNP-hard}\\\hline
\textcolor[rgb]{0.00,0.00,1.00}{$|Ag|+|R|$} & FPT & \textcolor[rgb]{1.00,0.00,0.00}{FPT} & \textcolor[rgb]{1.00,0.00,0.00}{FPT} & \textcolor[rgb]{1.00,0.00,0.00}{FPT} & \textcolor[rgb]{1.00,0.00,0.00}{FPT} & \hspace{10mm}?\\\hline
    \end{tabular}
\end{center}
\caption{Summary of results}
\label{table}
\end{figure*}

We summarize all the results in Figure~\ref{table}. The results from \cite{woolbridge-crg} are in green, from \cite{kraus} in
black and our results are in red color. We use the abbreviations NPC for NP-complete, and pNP for para-NP.

\section{Problems Left Open in Shrot et al.~\cite{kraus}}
\label{answer-to-open-problems-from-kraus}
First we show that SC parameterized by $|C|$ is W[1]-hard.

\begin{theorem}
\label{SC}
SC is W[1]-hard when parameterized by $|C|$.
\end{theorem}
\begin{proof}
We prove this by reduction from Independent Set (parameterized by size of independent set) which is a well-known W[1]-complete
problem. Let $H=(V,E)$ be a graph with $V = \{x_1,\ldots,x_n\}$ and $E = \{e_1,\ldots,e_m\}$. Let $k$ be a given integer. We
also assume that H has no isolated points as we can just add those points to the independent set and decrease the parameter
appropriately. We build a CRG $\Gamma$ as follows:
$$ \Gamma = \langle Ag,G,R,G_1,G_2,\ldots,G_k,\textbf{en},\textbf{req} \rangle $$
where\begin{itemize}
       \item $Ag = \{c_1,\ldots,c_k\}$
       \item $G_i = \{g^{1}_i,\ldots,g^{n}_i\}$ for all $i\in [k]$
       \item $G = \bigcup_{i=1}^{k} G_i$
       \item $R = \{r_1,\ldots,r_m\}$
       \item For all $i\in [k],j\in [m]$ , $\textbf{en}(c_i,r_j) = 1$
       \item For all $i\in [k],j\in [m]$ and $\ell\in [n]$, we have $\textbf{req}(g^{\ell}_i,r_j) = k$ if $e_j$ and $x_{\ell}$ are incident in $H$ and $\textbf{req}(g^{\ell}_i,r_j) = 0$ otherwise
     \end{itemize}

We claim that $H$ has an independent set of size $k$ if and only if the grand coalition $Ag$ is successful in $\Gamma$.\\

Suppose \textsc{Independent Set} answers YES, i.e., $H$ has an independent set of size $k$ say $I =
\{x_{\beta_1},\ldots,x_{\beta_k}\}$. Consider the goal set given by $G' = \{g^{\beta_1}_1,\ldots,g^{\beta_k}_k\}$. Clearly
$G'$ satisfies $Ag$ as $g^{\beta_i}_i \in G_i$ for all $i\in [k]$. Now consider any edge $e_j\in E(H)$. Let $\lambda$ be the
number of vertices from $I$ incident on $e_j$. Clearly $2\geq \lambda$ but as $I$ is independent set we have $1\geq \lambda$.
Now, for every $j\in [m]$ we have $\textbf{req}(G',r_j) = k\lambda \leq k = \textbf{en}(Ag,r_j)$. Thus $G'$ is feasible for
$Ag$. Summing up, $G'$ is successful for $Ag$ and hence SC answers YES for $C=Ag$.

Suppose now that SC answers YES for $C=Ag$. Let $G''\neq \emptyset$ be successful for $Ag$. Claim is that both
$g^{\beta}_i$ and $g^{\beta}_j$ cannot be in $G''$ if $i\neq j$. To see this, let $e_{\ell}$ be any edge incident
on $x_{\beta}$ (we had assumed earlier that graph has no isolated vertices). Then
$\textbf{req}(G'',r_{\ell}) \geq \textbf{req}(g^{\beta}_i,r_{\ell}) + \textbf{req}(g^{\beta}_i,r_{\ell}) = 2k > k = \textbf{en}(Ag,r_{\ell})$
which contradicts the fact that $G''$ is successful for $Ag$. Since $G_{i}$'s are disjoint and $G''$ is successful (hence also satisfiable) for $Ag$, we know that $G''$ contains at least one goal from each $G_i$. Also we have seen before that $g^{\beta}_i,g^{\gamma}_j \in G''$ and $i\neq j$ implies that $\beta \neq \gamma$. From each $G_i$ we pick any goal that is in $G''$. Let us call this as $G' = \{g^{\beta_1}_1,\ldots,g^{\beta_k}_k\}$. We know that $\beta_i \neq \beta_j$ when $i\neq j$. We claim that $I = \{x_{\beta_1},\ldots,x_{\beta_k}\}$ is an independent set in $H$. Suppose not and let $e_l$ be an edge between $x_{\beta_i}$ and $x_{\beta_j}$ for some $i,j\in [k]$. Then $\textbf{req}(G'',r_{\ell}) \geq \textbf{req}(G',r_{\ell}) \geq \textbf{req}(g^{\beta_i}_i,r_{\ell}) + \textbf{req}(g^{\beta_j}_j,r_{\ell}) = k + k > k = \textbf{en}(Ag,r_{\ell})$ which contradicts the fact that $G''$ is successful for $Ag$. Thus $I$ is an independent set of size $k$ in $H$ and so \textsc{Independent Set} also answers YES.\\

Note that $|Ag| = k, |G| = nk, |R| = m$ and so this reduction shows that the SC problem is W[1]-hard.
\end{proof}

We note that the SC problem can be solved in $O(|G|^{|C|}\times |R|)$ time (since we only need to check the subsets of size at
most $|C|$ of $G$) and thus SC parameterized by~$|C|$ is not para-NP-hard. Now we answer the only remaining open problem by
Shrot et al. by showing that ESCK parameterized by~$|R|$ is para-NP-hard.

\begin{theorem}
\label{esck} Checking whether there exists a successful coalition of size~$k$ (ESCK) is para-NP-hard when parameterized
by~$|R|$.
\end{theorem}
\begin{proof}
We prove this by reduction from SC which was shown to be para-NP-hard with respect to the parameter~$|R|$ in Theorem~3.8 of~\cite{kraus}.\\
Let $(\Gamma,C)$ be a given instance of SC. We consider an instance $(\Gamma',k)$ of ESCK
\begin{itemize}
       \item $Ag' = C $
       \item $R' = R$
       \item $G_{i}' = G_i$ for all $i\in C$
       \item $k = |C|$
\end{itemize}
We claim that SC answers YES if and only if ESCK answers YES.

Suppose SC answers YES, i.e., $C$ is a successful coalition in $\Gamma$. In $\Gamma'$ we just remove all agents not belonging
to $C$ from $\Gamma$. All the resources and the \textbf{en} and \textbf{req} functions carry over. So $C$ is a successful
coalition for $\Gamma'$ also. But we had chosen $k = |C|$ and so ESCK answers YES.

Suppose that ESCK answers YES. So there exists a successful coalition of size $k$ in $\Gamma'$. But $Ag' = C$ and we had
chosen $k = |C|$ and so the only coalition of size $k$ in $\Gamma'$ is the grand coalition $C = Ag'$. As ESCK answered YES we
know that $C$ is successful in $\Gamma'$. So it is also successful in $\Gamma$ and so SC also answers YES.

Note that $|Ag'| = k, |G'| = |G|, |R'| = |R|$ and so this reduction shows that the ESCK problem is para-NP-hard.
\end{proof}

\section{Problems Related to Resources}
\label{resources} For a coalition $C$, we recollect the notation we use: $ succ(C) = \{G'\ |\ G'\subseteq G\ ;\ G'\neq
\emptyset$ and $G'$ both satisfies $C$ and is feasible for it$\}$. In this section we show hardness results for three
different problems related to resources.

\subsection{Necessary Resource (NR)}
The idea of a \emph{necessary resource} is similar to that of a veto player in the context of conventional coalition games. A
resource is said to be \emph{necessary} if the accomplishment of any set of goals which is successful for the coalition would
need a non-zero consumption of this resource. Thus if a necessary resource is scarce then the agents possessing the resource
become important. We consider the \textsc{Necessary Resource} problem: Given a coalition $C$ and a resource $r$ answer YES if
and only if $\textbf{req}(G',r)>0$ for all $G'\in succ(C)$. NR was shown to be co-NP-complete in Wooldridge et al.
\cite{woolbridge-crg}. We note that if $C$ is not successful, then NR vacuously answers YES. We give a reduction from SC to
$\overline{NR}$.
\begin{lemma}
\label{lem:nr} Given an instance $(\Gamma,C)$ of SC we can construct an instance $(\Gamma',C',r')$ of NR such that SC answers
YES iff NR answers NO.
\end{lemma}
\begin{proof}
We keep everything the same except $R' = R \cup \{r'\}$. We extend the $\textbf{en}$ and $\textbf{req}$ functions to $r'$ by
$\textbf{en}(i,r') = 1$ for all $i\in Ag$ and $\textbf{req}(g,r') = 0$ for all $g\in G$. Now claim is that SC answers YES iff
NR answers NO.

Suppose SC answers YES. So $\exists \ G'\neq \emptyset$ such that $G'\in succ_{\Gamma}(C)$. Now $C\neq \emptyset$ and so
$\textbf{en}(C,r') > 0 = \textbf{req}(G',r') $ and thus $G'\in succ_{\Gamma'}(C)$. But $\textbf{req}(G',r') = 0$ and so NR
answers NO.

Suppose NR answers NO. So $succ_{\Gamma'}(C)\neq \emptyset$ as $\exists\ G'\in succ_{\Gamma'}(C)$ such that $G'\neq \emptyset$
and $\textbf{req}(G',r') = 0$. Now $\Gamma'$ is obtained from $\Gamma$ by only adding a new resource and so clearly $G'\in
succ_{\Gamma}(C)$. Thus SC will answer YES.
\end{proof}

\begin{theorem}
The parameterized complexity status of Necessary Resource is as follows :
\begin{itemize}
\item FPT when parameterized by $|G|$
\item co-W[1]-hard when parameterized by $|C|$
\item co-para-NP-hard when parameterized by $|R|$
\end{itemize}
\end{theorem}
\begin{proof}
When parameterized by $|G|$, we consider all $2^{|G|}$ subsets of $G$. For each subset, we can check in polynomial time if it
is a member of $succ(C)$ and if it requires non-zero quantity of the resource given in the input.

The other two claims follow from Lemma \ref{lem:nr}, Theorem 3.8 in Shrot et al., and Theorem \ref{SC}.
\end{proof}

\subsection{Strictly Necessary Resource (SNR)}
The fact that a resource is necessary does not mean that it will be used. Because the coalition in question can be
unsuccessful and hence the resource is trivially necessary. So we have the \textsc{Strictly Necessary Resource} problem: Given
a coalition $C$ and a resource $r$ answer YES if and only if $succ(C)\neq \emptyset$  and $\forall \ G'\in succ(C)$ we have
$\textbf{req}(G',r)>0$. SNR was shown to be strongly $D^{p}$-complete in Wooldridge et al. \cite{woolbridge-crg}. To prove the
parameterized hardness results, we give a reduction from SC to SNR.
\begin{lemma}
\label{lem:snr} Given an instance $(\Gamma,C)$ of SC we can construct an instance $(\Gamma',C',r')$ of SNR such that SC
answers YES iff SNR answers YES.
\end{lemma}
\begin{proof}
We keep everything the same except $R' = R \cup \{r'\}$. We extend the $\textbf{en}$ and $\textbf{req}$ functions to $r'$ by
$\textbf{en}(i,r') = |G|$ for all $i\in Ag$ and $\textbf{req}(g,r') = 1$ for all $g\in G$. Now claim is that SC answers YES
iff SNR answers YES.

We first show that $succ_{\Gamma}(C) = succ_{\Gamma'}(C)$. As $\Gamma'$ is obtained from $\Gamma$ by just adding one resource
and keeping everything else the same, we have $succ_{\Gamma'}(C) \subseteq succ_{\Gamma}(C)$. Now let $G_0\in
succ_{\Gamma}(C)$. Any coalition has at least one member and hence at least one $|G|$ endowment of resource $r'$. But
$\textbf{req}(G_0,r') = |G_0|\leq |G| \leq \textbf{en}(C,r')$ and so $G_0\in succ_{\Gamma'}(C)$. Summing up we have
$succ_{\Gamma}(C) = succ_{\Gamma'}(C)$.

Suppose SC answers YES. This implies $succ_{\Gamma}(C) \neq \emptyset$. So $succ_{\Gamma'}(C)= succ_{\Gamma}(C) \neq
\emptyset$. For every $G_0\in succ_{\Gamma'}(C), \textbf{req}(G_0,r') = |G_0| > 0$ as $G_0\neq \emptyset$. Therefore SNR
answers YES

Suppose SNR answers YES. So $succ_{\Gamma'}(C) \neq \emptyset$ as otherwise SNR would have said NO. Hence $succ_{\Gamma}(C) =
succ_{\Gamma'}(C) = \emptyset$ and SC so answers YES.
\end{proof}

\begin{theorem}
The parameterized complexity status of {Strictly Necessary Resource} is as follows :
\begin{itemize}
\item FPT when parameterized by $|G|$
\item W[1]-hard when parameterized by $|C|$
\item para-NP-hard when parameterized by $|R|$
\end{itemize}
\end{theorem}
\begin{proof}
When parameterized by $|G|$, we consider all $2^{|G|}$ subsets of $G$. For each subset, we can check in polynomial time if it
is a member of $succ(C)$ and if it requires non-zero quantity of the resource given in the input.

The other two claims follow from Lemma \ref{lem:snr}, Theorem 3.8 in Shrot et al., and Theorem \ref{SC}.
\end{proof}

\subsection{{$(C,G_0,r)$-Optimality} (CGRO)}
We may want to consider the issue of \emph{minimizing} usage of a particular resource. If satisfaction is the only issue, then
a coalition $C$ will be equally happy between any of the goal sets in $succ(C)$.  However in practical situations we may want
to choose a goal set among $succ(C)$ which minimizes the usage of some particular \emph{costly} resource. Thus we have the
\textsc{$(C,G_0,r)$-Optimality} problem: Given a coalition $C$, resource $r$ and a goal set $G_0\in succ(C)$ answer YES if and
only if $\textbf{req}(G',r)\geq \textbf{req}(G_0,r)$ for all $G'\in succ(C)$. CGRO was shown to be strongly co-NP-complete in
Wooldridge et al. \cite{woolbridge-crg}. To prove the parameterized hardness results, we give a reduction from SC to
$\overline{CGRO}$.
\begin{lemma}
\label{lem:cgro} Given an instance $(\Gamma,C)$ of SC we can construct an instance $(\Gamma',C',G_0,r')$ of CGRO such that SC
answers YES iff CGRO answers NO.
\end{lemma}
\begin{proof}
Define $G' = G \cup \{g'\}$, $R' = R \cup \{r'\}$ and $C' = C$. We extend the $\textbf{en}$ to $r'$ as follows:
$\textbf{en}(i,r') = 1$ for all $i\in C$ and $\textbf{en}(i,r') = 0$ if $i\notin C$. We extend $\textbf{req}$ to $g'$ and $r'$
as follows: $\textbf{req}(g',r') = |C|$, $\textbf{req}(g',r) = 0$ for all $r\in R$ and $\textbf{req}(g,r') = 0$ for all $g\in
G$. Let $G_0 = \{g'\}$. Now claim is that SC answers YES iff CGRO answers NO.

Suppose SC answers YES. So, $\exists \ G_1\in succ_{\Gamma}(C)$.  Claim is that $G_1\in succ_{\Gamma'}(C)$ because
$\textbf{en}(C,r') = |C| > 0 = \textbf{req}(G_1,r')$ as $G_1\subseteq G$. Note also that $G_0 = \{g'\}\in succ_{\Gamma'}(C)$
as $\textbf{en}(C,r') = |C| = \textbf{req}(G_0,r')$ and for every $r\in R$, $\textbf{en}(C,r) \geq 0 = \textbf{req}(G_0,r)$.
Therefore $\textbf{req}(G_1,r') = 0 < |C| = \textbf{req}(G_0,r')$ and hence CGRO answers NO.

Suppose CGRO answers NO. So $\exists \ G_1\in succ_{\Gamma'}(C)$ such that $\textbf{req}(G_1,r') < \textbf{req}(G_0,r') =
|C|$. Claim is $g'\notin G_1$ otherwise $\textbf{req}(G_1,r')\geq \textbf{req}(g',r') = |C|$. So $G_1\subseteq G$ and we
already had $G_1\in succ_{\Gamma'}(C)$. Therefore $G_1\in succ_{\Gamma}(C)$ and so SC answers YES.
\end{proof}

\begin{theorem}
The parameterized complexity status of {$(C,G_0,r)$-Optimality} is as follows :
\begin{itemize}
\item FPT when parameterized by $|G|$
\item co-W[1]-hard when parameterized by $|C|$
\item co-para-NP-hard when parameterized by $|R|$
\end{itemize}
\end{theorem}
\begin{proof}
When parameterized by $|G|$, we consider all $2^{|G|}$ subsets of $G$. For each subset, we can check in polynomial time if it
is a member of $succ(C)$ and if it requires atleast $\textbf{req}(G_0,r')$ quantity of resource $r'$ where $G_0$ and $r'$ are
given in the input.

The other two claims follow from Lemma \ref{lem:cgro}, Theorem 3.8 in Shrot et al., and Theorem \ref{SC}.
\end{proof}

\section{Problems Related to Resource Bounds}
\label{resource-bounds}
\subsection{{R-Pareto Efficient Goal Set} (RPEGS)}
We use the idea of \emph{Pareto Efficiency} to measure the optimality of a goal set w.r.t the set of all resources. In our
model we say that a goal set~$G'$ is \emph{R-Pareto Efficient} w.r.t a coalition~$C$ if no goal set in $succ_{\Gamma}(C)$
requires at most as much of every resource and strictly less of some resource. More formally we say that a goal set $G'$ is
\emph{R-Pareto Efficient} w.r.t a coalition~$C$ if and only if $\forall \ G''\in succ_{\Gamma}(C)$,
\begin{center}
$ \exists \ r_1\in R : \textbf{req}(G'',r_1)<\textbf{req}(G',r_1) \ \Rightarrow \ \exists \ r_2\in R :
\textbf{req}(G'',r_2)>\textbf{req}(G',r_2)$
\end{center}
We note that $G'$ is not necessarily in $succ(C)$. Thus we have the \textsc{R-Pareto Efficient Goal Set} problem: Given a
coalition $C$ and a goal set $G_0$ answer YES if and only if $G_0$ is \emph{R-Pareto Efficient} w.r.t $C$. Wooldridge et al.
\cite{woolbridge-crg} show that RPEGS is strongly co-NP-complete. To prove the parameterized hardness results, we give a
reduction from SC to $\overline{RPEGS}$.
\begin{lemma}
\label{lem:rpegs} Given an instance $(\Gamma,C)$ of SC we can construct an instance $(\Gamma',C',G_0)$ of RPEGS such that SC
answers YES iff RPEGS answers NO.
\end{lemma}
\begin{proof}
Define $R' = R \cup \{r'\}, G' = G \cup \{g'\}$ and $C' = C$. We extend the $\textbf{en}$ to $r'$ as follows:
$\textbf{en}(i,r') = |G|$ for all $i\in C$ and $\textbf{en}(i,r') = 0$ if $i\notin C$. We extend $\textbf{req}$ to $r'$ as
follows: $\textbf{req}(g,r') = |C|$ for all $g\in G$; $\textbf{req}(g',r') = |G|\cdot|C|+1$ and $\textbf{req}(g',r) =
\infty$ for all $r\in R$. Let $G_0 = \{g'\}$. Now claim is that SC answers YES iff RPEGS answers NO.

We first show that $succ_{\Gamma}(C) = succ_{\Gamma'}(C)$. Let $G_1\in succ_{\Gamma'}(C)$. Then claim is that $g'\notin
G_1$ because otherwise for all $r\in R$ we have $\textbf{req}(G_1,r)\geq \textbf{req}(g',r) = \infty > |G|\cdot|C| =
\textbf{en}(C,r)$. Also claim is that any goal set $G_2$ in $succ_{\Gamma}(C)$ also is in $succ_{\Gamma'}(C)$. All other
things carry over from $\Gamma$ and we have additionally that $\textbf{req}(G_2,r') = |G_2|\cdot|C| \leq |G|\cdot|C| =
\textbf{en}(C,r')$ as $G_2\subseteq G$. Hence we have $succ_{\Gamma}(C) = succ_{\Gamma'}(C)$.

Suppose SC answers YES, i.e., $\exists \ G_1\in succ_{\Gamma}(C)$. As $succ_{\Gamma}(C) = succ_{\Gamma'}(C)$ we have
$G_1\in succ_{\Gamma'}(C)$. Now for every $r\in R$, $\textbf{req}(G_1,r) < \infty = \textbf{req}(G_0,r)$. Also
$\textbf{req}(G_1,r') = |G_1|\cdot|C| \leq |G|\cdot|C| < |G|\cdot|C|+1 = \textbf{req}(G_0,r')$. Therefore $G_0$ requires
strictly more of every resource in $R'$ than $G_1$ and hence RPEGS answers NO.

Suppose RPEGS answers NO. Claim is that $succ_{\Gamma'}(C)\neq \emptyset$ otherwise it would have answered YES vacuously.
As $succ_{\Gamma}(C) = succ_{\Gamma'}(C)$ we have $succ_{\Gamma}(C)\neq \emptyset$ and hence SC answers YES.
\end{proof}

\begin{theorem}
The parameterized complexity status of {R-Pareto Efficient Goal Set} is as follows :
\begin{itemize}
\item FPT when parameterized by $|G|$
\item co-W[1]-hard when parameterized by $|C|$
\item co-para-NP-hard when parameterized by $|R|$
\end{itemize}
\end{theorem}
\begin{proof}
When parameterized by $|G|$, we consider all $2^{|G|}$ subsets of $G$. For each subset, we can check in polynomial time if it
is a member of $succ(C)$ and if it shows that $G_0$ is not \emph{R-Pareto Efficient}.

The other two claims follow from Lemma \ref{lem:rpegs}, Theorem 3.8 in Shrot et al., and Theorem \ref{SC}.
\end{proof}

\subsection{{Successful Coalition with Resource Bound} (SCRB)}
In real-life situations we typically have a bound on the amount of each resource. A \emph{resource bound} is a
function $\textbf{b} : R \rightarrow \mathbb{N}$ with the interpretation that each coalition has at most
$\textbf{b}(r)$ quantity of resource~$r$ for every $r\in R$. We say that a goal set $G_0$ \emph{respects} a resource
bound \textbf{b} w.r.t. a given CRG $\Gamma$ iff $\forall \ r\in R $ we have $\textbf{b}(r)\geq \textbf{req}(G_0,r)$.
Thus we have the \textsc{Successful Coalition With Resource Bound} problem: Given a coalition $C$ and a resource bound
~\textbf{b} answer YES if and only if $\exists \ G_0\in succ(C)$ such that $G_0$ respects~\textbf{b}. Wooldridge et
al. \cite{woolbridge-crg} show that SCRB is strongly NP-complete. To prove the parameterized hardness results, we give
a reduction from SC to $\overline{SCRB}$.
\begin{lemma}
\label{lem:scrb} Given an instance $(\Gamma,C)$ of SC we can construct an instance $(\Gamma',C',\textbf{b'})$ of SCRB such
that SC answers YES if and only if SCRB answers NO.
\end{lemma}
\begin{proof}
Define $R' = R \cup \{r'\}$ and $C' = C$. Let \textbf{b} be a vector with $|R'|$ components whose first $|R'|-1$
entries are 1 and the last entry is $|C|-1$, i.e., $\textbf{b} = \{1,1,\ldots,1,1, |C|-1\}$. We extend the
$\textbf{en}$ to $r'$ as follows: $\textbf{en}(i,r') = |G|$ for all $i\in C$ and $\textbf{en}(i,r') = 0$ if $i\notin
C$. We extend $\textbf{req}$ to $r'$ as follows: $\textbf{req}(g,r') = |C|$ for all $g\in G$. Now the claim is that SC
answers YES if and only if SCRB answers NO.

Suppose SC answers YES. So, there exists $G_0 \neq \emptyset$ such that $G_0 \in succ_{\Gamma}(C)$. In $\Gamma'$ we
have $\textbf{en}(C,r') = |G|\cdot|C|\geq \textbf{req}(G_0,r')$ as $G_0 \subseteq G$ and $\textbf{req}(g,r') = |C|$
for all $g\in G$. Thus $G_0 \in succ_{\Gamma'}(C)$ and so SCRB cannot vacuously answer YES. Now, for any $G''\in
succ_{\Gamma'}(C)$ such that $G''\neq \emptyset$ we have $\textbf{req}(G'',r')\geq |C| > |C| - 1 = \textbf{b}(r')$.
This means that no goal set in the non-empty set $succ_{\Gamma'}(C)$ respects \textbf{b} which implies that SCRB
answers NO.

Suppose SCRB answers NO. So $\exists \ G_0 \in succ_{\Gamma'}(C)$ such that $G_0 \neq \emptyset$ and $G_0$ respects
\textbf{b}. As $\Gamma'$ was obtained from $\Gamma$ by adding a resource and keeping everything else same, we have
$G_0\in succ_{\Gamma}(C)$ and hence SC answers YES.
\end{proof}

\begin{theorem}
The parameterized complexity status of {Successful Coalition With Resource Bound} (SCRB) is as follows:
\begin{itemize}
\item FPT when parameterized by $|G|$
\item co-W[1]-hard when parameterized by $|C|$
\item co-para-NP-hard when parameterized by $|R|$
\end{itemize}
\end{theorem}
\begin{proof}
When parameterized by $|G|$, we consider all $2^{|G|}$ subsets of $G$. For each subset,we can check in polynomial time if it
is a member of $succ(C)$ and if it requires non-zero quantity of the resource given in the input.

The other two claims follow from Lemma \ref{lem:scrb}, Theorem 3.8 in Shrot et al., and Theorem \ref{SC}.
\end{proof}

\section{Problems Related to Resource Conflicts}
\label{resource-conflicts}

\subsection{{Conflicting Coalitions} (CC)}
When two or more coalitions desire to use some scarce resource, it leads to a \emph{conflict} in the system. This
issue is a classic problem in distributed and concurrent systems. In our framework we say that two goal sets are
in \emph{conflict} w.r.t a resource bound if they are individually achievable within the resource bound but their
union is not. Formally a \emph{resource bound} is a function $\textbf{b} : R \rightarrow \mathbb{N}$ with the
interpretation that each coalition has at most $\textbf{b}(r)$ quantity of resource~$r$ for every $r\in R$. We say
that a goal set $G_0$ \emph{respects} a resource bound \textbf{b} w.r.t. a given CRG $\Gamma$ if and only if
$\forall \ r\in R $ we have $\textbf{b}(r)\geq \textbf{req}(G_0,r)$. We denote by $\emph{cgs}(G_1,G_2,\textbf{b})$
the fact that $G_1$ and $G_2$ are in conflict w.r.t $\textbf{b}$. Formally, $cgs(G_1,G_2,\textbf{b})$ is defined
as $\text{respects}(G_1,\textbf{b}) \wedge \text{respects}(G_2,\textbf{b}) \wedge \neg\text{respects}(G_1\cup
G_2,\textbf{b})$. Thus we have the \textsc{Conflicting Coalitions} problem: Given coalitions $C_1,C_2$ and a
resource bound~\textbf{b} answer YES if and only if $\forall\ G_1\in succ(C_1)$ and $\forall\ G_2\in succ(C_2)$ we
have $cgs(G_1,G_2,\textbf{b})$. Wooldridge et al. \cite{woolbridge-crg} show that CC is strongly co-NP-complete.
To prove the parameterized hardness results, we give a reduction from SC to $\overline{CC}$.

\begin{lemma}
\label{lem:cc} Given an instance $(\Gamma,C)$ of SC we can construct an instance $(\Gamma',C'_1,C'_2,\textbf{b})$ of CC such
that SC answers YES if and only if CC answers NO.
\end{lemma}
\begin{proof}
Define $R' = R \cup \{r'\}$ and $C'_1 = C = C'_2$. Let \textbf{b} be a vector with $|R'|$ components whose first $|R'|-1$
entries are $\infty$ and the last entry is $|G|\cdot|C|$, i.e., $\textbf{b} = \{\infty,\infty,\ldots,\infty,\infty,
|G|\cdot|C|\}$. We extend the $\textbf{en}$ to $r'$ as follows: $\textbf{en}(i,r') = |G|$ for all $i\in C$ and
$\textbf{en}(i,r') = 0$ if $i\notin C$. We extend $\textbf{req}$ to $r'$ as follows: $\textbf{req}(g,r') = |C|$ for all $g\in
G$. Now the claim is that SC answers YES if and only if CC answers NO.

First we claim that $succ_{\Gamma}(C) = succ_{\Gamma'}(C)$. We built $\Gamma'$ from $\Gamma$ by just adding one resource and
so clearly $succ_{\Gamma'}(C) \subseteq succ_{\Gamma}(C)$. Now let $G''\in succ_{\Gamma}(C)$. Then $\textbf{req}(G'',r') =
|G''|\cdot|C| \leq |G|\cdot|C| = \textbf{en}(C,r')$ and $G''\in succ_{\Gamma'}(C)$. Summarizing we have our claim.

Suppose SC answers YES. So, there exists $G_0 \neq \emptyset$ such that $G_0 \in succ_{\Gamma}(C)$. As
$succ_{\Gamma}(C) = succ_{\Gamma'}(C)$ we have $G_0 \in succ_{\Gamma'}(C)$. As $C'_1 = C = C'_2$ the \emph{cgs}
condition fails for $G_1 = G_0 = G_2$ and so CC answers NO.

Suppose CC answers NO. Claim is that $succ_{\Gamma'}(C) \neq \emptyset$. If not then $succ_{\Gamma'}(C'_1) =
\emptyset = succ_{\Gamma'}(C'_2)$ and in fact CC would have vacuously answered YES. But $succ_{\Gamma}(C) =
succ_{\Gamma'}(C)$ and so $succ_{\Gamma}(C)\neq \emptyset$. Thus SC answers YES.
\end{proof}

\begin{theorem}
The parameterized complexity status of {Conflicting Coalitions} (CC) is as follows :
\begin{itemize}
\item FPT when parameterized by $|G|$
\item co-W[1]-hard when parameterized by $|C|$
\item co-para-NP-hard when parameterized by $|R|$
\end{itemize}
\end{theorem}
\begin{proof}
When parameterized by $|G|$, we consider all $2^{|G|}$ choices for $G_1$ and $G_2$. Given a choice $(G_1,G_2)$ we can check in
polynomial time if $G_1$ and $G_2$ are members of $succ(C_1)$ and $succ(C_2)$ respectively. Also we can check the condition
$cgs(G_1,G_2,\textbf{b})$ in polynomial time.

The other two claims follow from Lemma \ref{lem:cc}, Theorem 3.8 in Shrot et al., and Theorem \ref{SC}.
\end{proof}

\section{The Parameter $|Ag|+|R|$ : Case of Bounded Agents plus Resources}
\label{ag+r} Considering the results in previous sections, we can see that even in the case that size of coalition or number
of resources is bounded the problem still remains computationally hard. So a natural question is what happens if we have a
bound on $|Ag|+|R|$ ? Can we do better if total number of agents plus resources is bounded? Shrot et.al \cite{kraus} show that
by this parameterization the problems SC, MAXC and MAXSC have FPT algorithms and they left the corresponding question for the
ESCK open. We will generalize the integer program given in Theorem 3.1 of \cite{kraus}, to give a FPT algorithm for the open
problem of Existence of Successful Coalition of size $k$ (ESCK). Then by using a similar approach we will design FPT
algorithms for the four other problems (NR, SNR, CGRO, SCRB) considered in this paper.

The integer program we define is a satisfiability problem (rather than an optimization problem). It consists
of a set of constraints, and the question is whether there exists an integral solution to this set. Consider
the following integer program (which we will name as IP):

\begin{alignat}{2}
    & \forall i \in Ag:                 & \sum_{g\in G_i} \Mb{x_g} \geq \Mb{y_i} \tag{1} \\
    & \forall r \in R:                  & \sum_{g\in G} \Mb{x_g}\times req(g,r) \leq \sum_{i\in Ag} \Mb{y_i}\times en(i,r) \tag{2} \\
    & \forall g \in G:                  & \Mb{x_g}\in \{0,1\} \notag \\
    & \forall i \in Ag:                 & \Mb{y_i}\in \{0,1\} \notag
\end{alignat}

In this setting, $\Mb{y_i}=1$, for each $i\in Ag$, represents the situation that the agent $i$ is participating in the
coalition and $\Mb{x_g}=1$, for each $g \in G$, represents the situation that goal $g$ is achieved. The first constraint
guarantees that any participating agent has at least one of his goals achieved. The second constraint ensures that the
participating agents have enough endowment to achieve all of the chosen goals. It is clear that any solution for this integer
program is a coalition of agents and a successful set of goals for that coalition.
~\\

The above integer program has $|Ag|+|R|$ constraints and in Flum and Grohe~\cite{flum-grohe} it is shown that checking
feasibility of \textsc{Integer Linear Programming} is FPT in the number of constraints or in the number of variables. Now for
each of our problems we will add some constraints to get new integer programs which solve those problems.

\begin{theorem} \label{AR:ESCK}
Checking whether there is a Successful Coalition of size $k$ (ESCK) is FPT when parameterized by $|Ag|+|R|$.
\end{theorem}
\begin{proof}
For ESCK, the general integer program given above needs only one additional constraint: We have to ensure that exactly $k$
number of agents will be selected. Therefore adding the constraint $\sum_{i\in Ag} \Mb{y_i} = k$ gives us the integer program
for the problem ESCK. The number of constraints, i.e., $|Ag|+|R|+1$ for this integer program is $|Ag|+|R|+1$ and as
\textsc{Integer Linear Programming} is FPT w.r.t number of variables or constraints we have that ESCK parameterized by
$|Ag|+|R|$ is FPT.

\end{proof}

In the problems NR, SNR and CGRO the coalition $C$ is always given. So we will change the variables $\Mb{y_i}$'s to constants
where $\Mb{y_i}=1$ if $i\in C$ and 0 otherwise. We call this new integer program a \textbf{Fixed Coalition Integer Program}
(FCIP). The coalition $C$ is successful if and only if FCIP is satisfiable.

\begin{theorem} \label{AR:NR}
Checking whether the Resource $r$ is Needed for a Coalition $C$ to be Successful (NR) is FPT when
parameterized by $|Ag|+|R|$.
\end{theorem}
\begin{proof}
We start with the integer program FCIP. The answer to NR is YES, if and only if in any successful subset of goals, there is at
least one goal $g$ with $\Mb{req}(g,r)>0$. So we just need to check and see if the coalition is successful by only using the
goals which do not need the resource $r$. Therefore in FCIP, for all goals $g\in G$ where $\Mb{req}(g,r)>0$ we will set the
variable $x_g$ to zero. Now the answer to NR is YES iff the resulting integer program is not satisfiable. Note that the number
of constraints is still same as previously - $|Ag|+|R|$.  As \textsc{Integer Linear Programming} is FPT wrt number of
variables or constraints we have that NR parameterized by $|Ag|+|R|$ is FPT.

\end{proof}

\begin{theorem} \label{AR:SNR}
Checking whether the Resource $r$ is Strictly Needed for a Coalition $C$ to be Successful (SNR) is FPT when
parameterized by $|Ag|+|R|$.
\end{theorem}
\begin{proof}
We start with the integer program FCIP. Since SNR answers NO when the coalition is not successful, we should first check if
the coalition is successful. Therefore we will check the answer to FCIP. If it is not satisfiable, then the answer for SNR
would be NO. But if FCIP is satisfiable, i.e., $succ(C)\neq \emptyset$, then we just need to check and see if the coalition is
successful by only using the goals which do not need the resource $r$. Again with the same approach as the proof of Theorem
\ref{AR:NR}, for all goals $g\in G$ where $\Mb{req}(g,r)>0$ we will set the variable $x_g$ to zero. Now the answer to SNR is
YES iff the resulting IP is not satisfiable. Note that the number of constraints is still same as previously - $|Ag|+|R|$. As
\textsc{Integer Linear Programming} is FPT w.r.t number of variables or constraints we have that SNR parameterized by
$|Ag|+|R|$ is FPT.

\end{proof}

\begin{theorem} \label{AR:CGRO}
Checking whether the successful goal set $G_0$ has optimal usage of Resource $r$ for a Coalition $C$ (CGRO) is
FPT when parameterized by $|Ag|+|R|$.
\end{theorem}
\begin{proof}
We start with the integer program FCIP. The limit on usage of resource $r$ is $\Mb{req}(G_0,r)$. Let $\beta = \Mb{req}(G_0,r)$
be the limit. So the answer for CGRO is YES iff there is no successful set of goals $G'$ with $\Mb{req}(G',r)<\beta$. So by
adding the constraint $\sum_{g\in G} \Mb{x_g}\times req(g,r) < \beta$ to FCIP, the answer for CGRO would be YES iff the
resulting IP is not satisfiable. Note that the number of constraints $|Ag|+|R|+1$. As \textsc{Integer Linear Programming} is
FPT w.r.t number of variables or constraints we have that CGRO parameterized by $|Ag|+|R|$ is FPT.

\end{proof}


\begin{theorem}
Checking whether a given coalition~$C$ is Successful by respecting the Resource Bound $\Mb{b}$ (SCRB) is FPT
when parameterized by $|Ag|+|R|$.
 \label{ar-scrb}
\end{theorem}
\begin{proof}
We start with the integer program FCIP. Now the limit on usage of any resource $r\in R$ is $\Mb{b}(r)$. So for every resource
$r\in R$ we will bound its usage by adding the constraint $\sum_{g\in G} \Mb{x_g}\times \textbf{req}(g,r)\leq \Mb{b}(r)$ to
FCIP. Now the answer for SCRB would be YES if and only if the resulting integer program is satisfiable. Note that number of
constraints now is $|Ag| + 2|R|$ and \textsc{Integer Linear Programming} is FPT w.r.t number of variables or constraints we
have that SCRB is FPT w.r.t $|Ag|+ 2|R|$ and hence wrt $|Ag|+|R|$.

\end{proof}

\begin{theorem}
Checking whether a given goal set $G_0$ is \emph{R-Pareto Efficient} (RPEGS) is FPT when parameterized by
 $|Ag|+|R|$. \label{ar-rpegs}
\end{theorem}
\begin{proof}
As in the proof of Theorem~\ref{ar-scrb}, set the variables $\Mb{y_i}=1$ if $i\in C$ and $0$ otherwise. The answer for the
problem is NO if there exists a successful $G'$ such that there is a resource~$r\in R$ such that $\textbf{req}(G', r) <
\textbf{req}(G_0, r)$ and for every other resource~$r'\in R$ we have $\textbf{req}(G', r') \leq \textbf{req}(G_0, r')$. Since
$G_0$ is given, $\textbf{req}(G_0, r)$ is a constant. So we can write $|R|$ IPs, such that in the integer program for the
resource~$r$, we have the constraint $\textbf{req}(G', r) < \textbf{req}(G_0, r)$, and $|R|-1$ constraints $\textbf{req}(G',
r') \leq \textbf{req}(G_0, r')$, one for each resource $r' \neq r$. Now the answer for RPEGS would be YES iff all $|R|$
integer programs are not satisfiable. Note that the number of constraints in each of the integer programs is $|Ag| + 2|R|$ and
\textsc{Integer Linear Programming} is FPT w.r.t number of variables or constraints we have that RPEGS is FPT w.r.t
$|Ag|+2|R|$ and hence w.r.t $|Ag|+|R|$.

\end{proof}

Now we give an \textbf{integer quadratic program} for the CC problem :

\begin{alignat}{2}
    \forall i \in Ag:                 &\qquad& \sum_{g\in G_i} \Mb{x_g} &\geq \Mb{y_i} \tag{1} \\
	\forall r \in R:                  && \sum_{g\in G} \Mb{x_g}\times \textbf{req}(g,r) &\leq \sum_{i\in Ag} \Mb{y_i}\times \textbf{en}(i,r) \tag{2} \\
    \forall g \in G:                  && \Mb{x_g}&\in \{0,1\} \notag \\
    \forall i \in Ag:                 && \Mb{y_i}&\in \{0,1\} \notag
\end{alignat}

\begin{alignat}{2}
    \forall i \in Ag:                 &\qquad& \sum_{g\in G_i} \Mb{X_g} &\geq \Mb{Y_i} \tag{3} \\
	\forall r \in R:                  && \sum_{g\in G} \Mb{X_g}\times \textbf{req}(g,r) &\leq \sum_{i\in Ag} \Mb{Y_i}\times \textbf{en}(i,r) \tag{4} \\
    \forall g \in G:                  && \Mb{X_g}&\in \{0,1\} \notag \\
    \forall i \in Ag:                 && \Mb{Y_i}&\in \{0,1\} \notag
\end{alignat}

In the first sub-program, we set $y_i = 1$ if $i\in C_1$ and 0 otherwise. Then this sub-program finds a goal
set $G_1\in succ_{\Gamma}(C_1)$. The second sub-program is similar. Now we add the resource bound conditions :
$$ \forall \ r\in R \ \ \ \sum_{g\in G} \Mb{x_g}\times \textbf{req}(g,r) \leq  \textbf{b}(r)$$
$$ \forall \ r\in R \ \ \ \sum_{g\in G} \Mb{X_g}\times \textbf{req}(g,r) \leq  \textbf{b}(r)$$
$$ \exists \ r\in R \ \ s.t. \ \ \ \sum_{g\in G} \Big(\Mb{x_g + X_g - x_g\cdot X_g}\Big)\times \textbf{req}(g,r) >  \textbf{b}(r)$$
The first two conditions state that both $G_1$ and $G_2$ respect $\textbf{b}$ and the third condition says
that $G_1\cup G_2$ does not respect $\textbf{b}$. However the above program is \textbf{quadratic} due to the
last constraint and there is no known result about fixed parameter tractability for quadratic integer
programs. Hence we leave open the question about status of CC parameterized by $|Ag|+|R|$.

\section{Revisiting ESCK Parameterized by $|G|$}
Shrot et al. \cite{kraus} show in Theorem 3.2 of their paper that ESCK parameterized by $|G|$ is FPT. We first show their
proposed FPT algorithm is wrong by giving an instance when their algorithm gives incorrect answer. Then we show that in fact
the problem is para-NP-hard via a reduction from the independent set problem.

\subsection{Counterexample to the Algorithm Given in Theorem 3.2 of Shrot et al. \cite{kraus}}
The algorithm is as follows:
\begin{enumerate}
  \item For each $G'\subseteq G$
        \begin{itemize}
          \item Let $C'$ be set of all agents satisfied by $G'$
          \item If $|C'|\neq k$ , go to 1.
          \item If $G'$ is feasible for $C'$, return TRUE
        \end{itemize}
  \item return FALSE
\end{enumerate}
We give an instance $\Gamma$ where the above algorithm gives an incorrect answer. Suppose $|Ag|>k$, each agent has 1 unit of
endowment of each resource, each goal requires 0 of each resource, and $G_i = G$ for all agents $i\in Ag$. Thus each coalition
is successful and $\forall \ G'\subseteq G$ we have $C' = Ag$ which means that $|C'| = |Ag| > k$ and so the algorithm answers
NO while the correct answer is YES. Indeed by  reducing {Independent Set} to a CRG instance with $|G|=1$, we prove the
following theorem.
\begin{theorem}
\label{esck}
ESCK parameterized by~$|G|$ is para-NP-hard.
\end{theorem}
\begin{proof}
We prove this by reduction from \textsc{Independent Set} to a CRG with $|G|=1$. Let $H=(V,E)$ be a given graph
and let $k$ be the given parameter. Let $V = \{v_1,v_2,\ldots,v_n\}$ and $E = \{e_1,e_2,\ldots,e_m\}$. We
build an instance $(\Gamma,k)$ of ESCK where
\begin{itemize}
       \item $Ag = \{a_1,a_2,\ldots,a_n\} $
       \item $R = \{r_1,r_2,\ldots,r_m\}$
       \item $G = \{g\}$
       \item $G_{i}= G \ \ \forall \ i\in Ag$
       \item $\textbf{req}(g,r_j) = k-1 \ \ \forall \ j\in [m]$
       \item $\textbf{en}(a_i,r_j) = 0$ if $v_i$ and $e_j$ are incident and 1 otherwise
\end{itemize}
We now claim that \textsc{Independent Set} answers YES if and only if ESCK answers YES.

Suppose \textsc{Independent Set} answers YES, i.e., $H$ has an independent set of size $k$ say $I =
\{v_{\beta_1},v_{\beta_2},\ldots,v_{\beta_k}\}$. Consider the following coalition of size $k$: $C =
\{a_{\beta_1},a_{\beta_2},\ldots,a_{\beta_k}\}$. Clearly the goal set $\{g\}$ is satisfying for $C$. Also, as
$I$ is independent set the number of vertices from $I$ incident on any $r_j$ is atmost 1. So $\forall\ j\in
[m]$ we have $\textbf{req}(g,r_j) = k-1 \leq \textbf{en}(C,r_j)$ and so $\{g\}$ is feasible for $C$ which
means that $C$ is successful coalition. As $|C|=k$ we have that ESCK answers YES.

Suppose that ESCK answers YES. So there exists a successful coalition of size $k$ in $\Gamma$ say $C =
\{a_{\beta_1},a_{\beta_2},\ldots,a_{\beta_k}\}$. Consider the set of vertices $I =
\{v_{\beta_1},v_{\beta_2},\ldots,v_{\beta_k}\}$ in $V$. We claim that it is an independent set. Suppose not
and let $e_j$ be an edge joining $v_{\beta_i}$ and $v_{\beta_l}$ such that $v_{\beta_l},v_{\beta_l}\in I$.
Then we have $\textbf{en}(C,r_j)\leq k-2 < k-1 = \textbf{req}(g,r_j)$ which contradicts the fact that $C$ is
successful (Since $G=\{g\}$ the only possible goal set is $\{g\}$). Therefore $I$ is an independent set and as
$|I|=k$ we have that \textsc{Independent Set} answers YES.

Note that as $|G|=1$ in our CRG $\Gamma$ and \textsc{Independent Set} is known to be NP-hard we have that ESCK parameterized
by $|G|$ is para-NP-hard.
\end{proof}

\section{Conclusions and Directions for Future Work}
We considered some of the problems regarding resources bounds and resource conflicts which were shown to be computationally
hard in \cite{woolbridge-crg} but were not considered in \cite{kraus}. We also solved 3 open questions posed in \cite{kraus}
by showing that
 \begin{enumerate}
   \item SC parameterized by $|C|$ is \emph{W[1]-hard}
   \item ESCK parameterized by $|Ag|+|R|$ is \emph{FPT}
   \item ESCK parameterized by $|R|$ is \emph{para-NP-hard}
 \end{enumerate}
We also found a bug in Theorem 3.2 of \cite{kraus} which claimed that ESCK parameterized by $|G|$ is FPT. We give a
counterexample to their algorithm and in fact show that the problem is para-NP-hard. Then for some problems related to
resources, resource bounds and resource conflicts like NR, SNR, CGRO, RPEGS, SCRB and CC we have results when parameterized by
various natural parameters like $|G|, |C|,|R|$ and $|Ag|+|R|$ (only CC parameterized by $|Ag|+|R|$ is left open).

These results help us to understand better the role of the various components of the input and identify
exactly the ones which make the input hard. Since all the problems are known to be FPT when parameterized by
$|G|$ and all of them except CC are known to be FPT when parameterized by $|Ag|+|R|$ we know that our problems
are tractable when the goal set is small. With this knowledge we can even want to enforce this restriction in
real-life situations as much as possible. On the other hand we know that all the problems we considered remain
intractable when parameterized by $|C|$ or $|R|$ and hence there is no point in trying to restrict size of
coalition or number of resources as it does not make the computation faster

The study of problems arising in coalitions of agents in multi-agents systems using the parameterized complexity paradigm was
initiated by Shrot et al.~\cite{kraus} In this paper we have tried to take a further step in this direction which we believe
is still unexplored. There are various (classically) computationally hard problems which need to be better analyzed through
the rich theory of parameterized complexity.

Both in Shrot et al.~\cite{kraus} and this paper only the CRG model has been considered. In CRG model the status of CC
parameterized by $|Ag|+|R|$ is left open. Alternatively one might consider other natural parameters like $|Ag|$ or try to
examine other models like the QCG model~\cite{woolridge-qcg} through parameterized complexity analysis.

\section{Acknowledgments}
We would like to thank Yuk Hei (Tom) Chan, Dana Nau and Kanthi Sarpatwar for helpful discussions.

\newpage
\bibliography{docsdb}

\end{document}